\documentclass{jicspack}

\setvolume{10}  
\setyear{2013}                        
\setpagerange{1}{8}                   
\setheadauthor{X. Liu et al.}          
\setissn{1548--7741}%
\setpubdate{January 2013}%
\setno{1}     

\setdoi{10.12733/jicsxxxxxxxx}

\afterpage{\beginheader}                   

\usepackage{amssymb}
\usepackage{subfigure}
\usepackage{amsmath}
\usepackage{bm}
\usepackage{multirow}
\usepackage{mathrsfs}
\usepackage{mathtools}
\usepackage{algorithm}
\usepackage{url}
\usepackage{algpseudocode}

\def\lim{\mbox{lim}}

\newcommand{\EqRef}[1]{Eq.~(\ref{#1})}

\newcommand{\ThemRef}[1]{Theorem~\ref{#1}}

\newcommand{\FigRef}[1]{Figure~\ref{#1}}
\newcommand{\AlgRef}[1]{Algorithm~\ref{#1}}
\newcommand{\TabRef}[1]{Table~\ref{#1}}

\newtheorem{theorem}{Theorem}
\newtheorem{definition}{Definition}
\newenvironment{proof}[1][Proof]{\begin{trivlist}
\item[\hskip \labelsep {\bfseries #1}]}{\end{trivlist}}
\usepackage{slashbox}  
\usepackage{array}

\begin{document}
\begin{premaker}

\title{On the Complexity of One-class SVM for Multiple Instance Learning}
\thanks[label1]{Project supported by the National Nature Science
Foundation of China (No.~***).}

\author[author1]{Zhen Hu\corauthref{cor1}}
\ead{49859211@qq.com}
\corauth[cor1]{Corresponding author.}
\author[author1]{Zhuyin Xue}

\address[author1]{Sci. $\&$ Technol. on Inf. Syst. Eng. Lab., Nanjing 210000, China}


\begin{abstract}

In  traditional multiple instance learning (MIL), both positive and
negative bags are required to learn a prediction function.
However, a high human cost is needed to know the label of each bag---positive or negative.
Only positive bags contain our focus (positive instances) while negative bags consist of noise or background (negative instances).
So we do not expect to spend too much to label the negative bags.
Contrary to our expectation, nearly all existing MIL methods require enough negative bags besides positive ones.
In this paper we propose an algorithm called ``Positive Multiple Instance'' (PMI), which learns a classifier given only a set of positive bags. So the annotation of negative bags becomes unnecessary in our method.
PMI is constructed based on the assumption that the unknown positive instances in positive bags be similar each other and constitute one compact cluster in
feature space and the negative instances locate outside this cluster.
The experimental results demonstrate that PMI achieves the
performances close to or a little worse than those of the traditional MIL algorithms on benchmark and real data sets. However, the number of training bags in PMI is reduced significantly compared with traditional MIL algorithms.
\end{abstract}
\end{premaker}


\section{Introduction}
Multiple instance learning (MIL) is introduced by \cite{dietterich1997solving} to solve the drug activity prediction problem. During the past years, MIL approaches have been applied successfully to
many real applications such as image categorization
\cite{chen2006miles,chen2004image,tang2010image}, image retrieval
\cite{zhang2002content,zhang2005multiple}, tracking
\cite{babenko2011robust}, web mining \cite{zafra2011multiple}, gene expression \cite{li2012drosophila}  and medical diagnosis
\cite{fung2007multiple}. In traditional supervised learning,
each instance(feature vector) of training set corresponds to one given label.
By contrast, every set of instances (not one instance) is associated with a given label in MIL. Each instance set is called a ``bag''. If a bag contains at least one positive instance, it is labeled as positive and negative otherwise. It is unknown which instance is positive in each bag. In other words, only bag level label is available while instance level label is not in positive bags. The task of MIL is to learn a
concept to predict the label of an unseen bag.


In nearly all existing MIL algorithms, both positive and negative bags are required during the training phase.
But only the positive instances are our focus generally and the negative instances are unrelated to our interest. For example, if we attempts to learn a concept of face, the face patches in an image are positive instances and other non-face images are labeled as negative bags. According to this preference, we are willing to concentrate on labeling face images and ignore non-face images as much as possible. Non-face images are not our interest.
However, negative bags are not less and even much more than positive in many real data sets such as Corel \cite{Corel}. The availability of negative bags requires a high cost.
So labeling on negative bags brings great inconvenience in real applications.
In addition, it takes more time to label
one negative bag than positive in that every instance must
be confirmed to be negative. In
contrast, remaining instances in one positive bag can be ignored
if only one positive instance is found. The labeling on negative bags increases label cost significantly. So it is necessary to design a method using only positive bags to learn.

One related work is \cite{zhang2005multiple}.
His method is simply to solve a query
problem, in which an extra positive instance must be provided by
user. The most similar instance in every positive bag to this provided positive instance
is considered positive. This solution has the following
disadvantages. Firstly, the positive instance is difficult to
provide or unavailable in some applications. So \cite{zhang2005multiple} can be applied to a limited range of settings.
Secondly, some bags contain more than one positive instances. \cite{zhang2005multiple} cannot select multiple positive instances accurately in one bag. Thirdly, the performance of this algorithm depends mostly upon the positive instance given by user. It causes
the prediction accuracy of \cite{zhang2005multiple} is too sensitive to the provided positive instance. In contrast, our proposed algorithm in this paper
requires no additional positive instance in most cases. Therefore, our proposed algorithm can solve a wider range of problems.

In this paper, a new algorithm called PMI(Positive Multiple
Instance) is proposed, which learns a concept from only positive training bags.
PMI is designed based on the assumption that the positive instances
constitute one compact cluster in feature space and most negative instances
locate outside this cluster. 
PMI works with two major steps---training and query.
The training step is to convert the problem of learning with only positive bags into a one-class classification problem \cite{taylor2000support} to get a classifier $f$. If the instance level label is not available, PMI terminates and outputs $f$ as the concept. Otherwise, go to the query step. The query step
is to select an instance to query its label from instances positively labeled by the
classifier $f$ in the previous training step. If the queried instance is negative,
remove the instances positively labeled by the classifier $f$ from the training bags and return to the training step.
Otherwise, PMI outputs $f$ as the desired concept and terminates.
We provide the maximum number of queried instances
theoretically. The queried instance number in real applications usually is
far smaller than the theoretical result.
The experimental results demonstrate that PMI achieves close
performances to those of the traditional MIL algorithms on most data sets.
Our contribution is that the training bags number can be reduced greatly at a little or no worse on accuracy.
Negative bags are unnecessary for training in our method.

The remainder of this paper is organized as follows. We introduce PMI algorithm in Section 2 in detail.
Section 3 illustrates PMI with experiments. We conclude on our work in Section 4 finally.

\section{Our Proposed Algorithm}
In this section, we propose PMI algorithm.
One-class SVM \cite{taylor2000support} is involved in PMI, so a
brief review on one-class SVM is presented here.

\subsection{Review on One-Class SVM}
The formal definition of one-class SVM is described as following. $N$
$d$-dimensional instances ${x_1, x_2, ... , x_N}$ in feature space $\mathbb{R}^d$ are given, where $\mathbb{R}$ denotes real number field.
The task is to learn a prediction function takes value +1 while capturing most given instances in a small region of feature space and -1 otherwise.
One-class SVM maximizes the margin between the
training instances and origin in feature space to obtain a decision
hyperplane by the following formulation:
\begin{equation}\label{one-class SVM}
\begin{array}{l}
\displaystyle    \min_{w,\rho,\xi,\nu}  \frac{1}{2}\|w\|^2 +
\frac{1}{\nu N}\sum_{i}\xi_i
    -\rho  \\[1.0ex]
\displaystyle    \mbox{s. t.}~~(w\cdot \Phi(x_i))\geq \rho - \xi_i,  \\[1.0ex]
\displaystyle    ~~~~~~\xi_i\geq 0, w\in \mathbb{R}^d, \rho\in \mathbb{R},   
\end{array}
\end{equation}
where $\nu\in (0,1)$ is a parameter to balance the regularization
$\|w\|^2$ and the hinge loss $\displaystyle \sum_{i}\xi_i$, $w$
and $\rho$ are the weight coefficient and the bias in linear
decision function respectively, $\xi, ~~\xi^T=[\xi_1,...,\xi_n]$ denotes a slack variable vector,
$\xi_i$ is the $i$-th element of $\xi$,
and $\Phi(x_i)$ is a nonlinear map function to $x_i$ to deal with nonlinear boundary of training instances. A kernel
function $ker(x, y)$ is defined to replace $\Phi(x)\cdot \Phi(y)$.
Model \ref{one-class SVM} is a quadratic programming problem, which can be solved through its dual form:
\begin{equation}\label{one-class SVM primal}
\begin{array}{l}
\displaystyle    \min_{\alpha}  \frac{1}{2}\sum_{i,j}\alpha_{i}\alpha_{j}ker(x_i,x_j)  \\[1.0ex]
\displaystyle    \mbox{s. t.}~~   0\leq \alpha_i \leq
\frac{1}{\nu N}, ~~\sum_i\alpha_i=1,
\end{array}
\end{equation}
where $\alpha_i$ is a Lagrange multiplying factor of $x_i$. The value of
$\alpha_i$ is summarized into three categories according to where $x_i$ locates:
\begin{enumerate}
  \item $\alpha_i=0$ $\rightarrow$ $x_i$ locates inside positive instance region;
  \item $\frac{1}{\nu N}>\alpha_i>0$$\rightarrow$ $x_i$ lies on the prediction function boundary;
  \item $\alpha_i=\frac{1}{\nu N}$$\rightarrow$ $x_i$ falls outside the positive instance region.
\end{enumerate}

The label of an unseen instance $x$ is predicted by the following function:
\begin{equation}\label{one-class SVM result}
 \displaystyle   f(x) = \mbox{sign}(\sum_{i}\alpha_i ker(x, x_i)-\rho),
\end{equation}
where the function $\mbox{sign}(y)$ outputs +1 if $y\geq 0$ and -1 otherwise. If
$f(x)\geq 0$, we assume that $x$ should be similar to the training instances with a high likelihood. Otherwise, $x$ is an outlier.
$\rho$ in \EqRef{one-class SVM result} is computed:
\begin{equation}\label{rho solution}
\displaystyle    \rho = \sum_{i}\alpha_i ker(x_i,x_j),
\end{equation}
where $x_j$ is a support vector, which implies that $x_j$ locates on the
boundary of decision function $f(x)$ and $f(x)=0$.

Various complex separation hyperplanes can be described by different
types of kernel functions, such as polynomial, RBF, sigmod or
self-defined ones. RBF kernel is given by:
\begin{equation}\label{RBF kernel}
    ker(x,y) = e^{\gamma\|x-y\|^2},
\end{equation}
where $\gamma$ denotes a parameter to control how similar $x$ is to $y$. In
general, we assume that samples of interest locate inside closed regions in the feature space. RBF function is one of the most widely used
kernel for its flexibility. We will use RBF as the default kernel
function in the rest of this paper.

\subsection{Our Proposed Algorithm}
Our proposed algorithm PMI is introduced in detail in this subsection.
At the beginning, we provide the formal definition of the problem PMI solves.
A collection of bags $\{B_1, B_2,
..., B_N\}$ are given as the training set. The $i$-th bag in the training set $B_i$
contains $N_i$ $d$-dimensional instances $B_i=[B_{i1},..., B_{i
N_i}]\in \mathbb{R}^{d\times N_i}$ and $B_{ij} \in \mathbb{R}^d$ is a $d$-dimensional instance, $j=1,...,N_i$. We
define an $d\times n$ matrix stacking all instances together $B=[B_1,...,B_{N}]$, where $n$ denotes the total number of all instances in $N$ bags. At least one
positive instance resides in each bag of training set, but it is unknown which instance is positive.
The same as the label rule in previous MIL approaches, a bag is positive if it contains at least one positive instance and is negative otherwise.
Our goal is to learn a function to predict the label of an unseen bag. PMI works in two steps: training and query.
If instance level label is available, PMI algorithm alternates between training and query step until the desired concept is obtained. Otherwise, PMI runs training step for once. The remainder describes the two steps explicitly.

\subsubsection{Training Step}
As is mentioned in Section Introduction, PMI assumes that positive instances be similar to each other in feature space. We explain this assumption in \FigRef{fourfaces}.
Take face identification as an example, the task is to learn a function to predict whether an image contains face or not. So the patches containing face are positive instances we are focus on in \FigRef{fourfaces} and the other non-face patches are negative instances. It can be easily found that face patches (red rectangles in \FigRef{fourfaces}) look similar to each other while non-face patches are dissimilar. In summary, positive instances refer to one concept leading to the high similarity between positive and negative instances usually locate outside the region occupied by positive instances in feature space.
\begin{figure}
  \centering
  \includegraphics[scale=0.5]{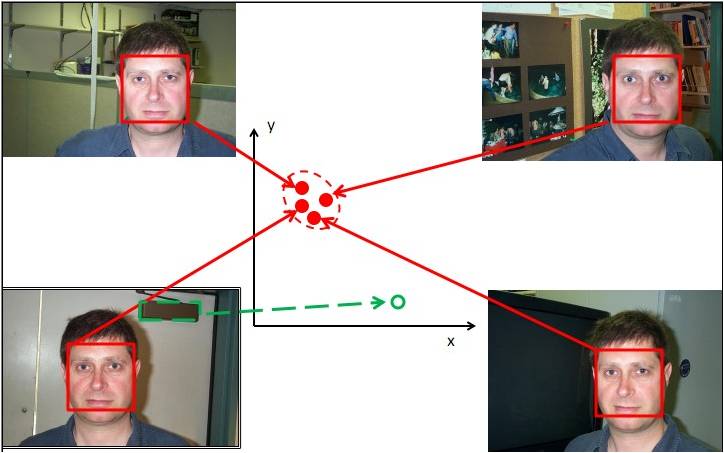}\\
  \caption{It is illustrated positive instances are clustered in a small region and negative instances locate outside. The red solid points are positive and the green circle is negative.}\label{fourfaces}
\end{figure}

According to the previous explanation and illustration in \FigRef{fourfaces}, we assume that most positive instances gather in a compact cluster and negative instances locate outside this cluster.
At least one instance in most training bags resides in this positive cluster (the cluster composed of positive instances).
Now we needs a function to describe the positive cluster. This function takes the value +1 in the positive cluster region and takes -1 elsewhere. We can apply this function on every instance in one bag to identify whether this bag is positive or not.
But since we actually do not which instance is positive, the positive concept cannot directly be learned from positive instances.


As is stated in \cite{andrews2003support,fung2007multiple}, the selection of
instances will be formulated as a combinatorial optimization problem,
which is non-convex and difficult to solve especially when a global
optimal solution is required. We tend to approximate one virtual instance $b_i$ with a high likelihood to be positive for each training bag $B_i$. The desired concept can be learned from these approximated positive instances.
To make our method efficient and easy to solve, a linear combination coefficient
vector $\lambda_i$ is used to convert each bag $B_i$ to a virtual positive
instance $b_i$ in feature space:
\begin{equation}\label{linear coefficient}
\begin{array}{l}
\displaystyle  ~~~~b_i =  B_i\lambda_i = \sum_{j=1}^{N_i}B_{ij}\lambda_{ij},  \\[1.0ex]
\displaystyle  \mbox{s.t.}~\sum_{j}\lambda_{ij}=1,~~\lambda_{ij} \geq
0,~~i=1,...,N,
\end{array}
\end{equation}
where $\lambda_i=[\lambda_{i1},...,\lambda_{iN_i}]^\top$.
Let $\lambda$ denote the vector concatenating all $\lambda_i$, $\lambda= [\lambda_1;
\lambda_2; ...; \lambda_N]$. Assume that $b_i$ is positive, $b_i$ should be highly close to the positive instances in $B_i$. Furthermore, if $B_{ij}$ is positive, $\lambda_{ij}$ will be assigned a larger value and small otherwise. According to the assumption that the positive instances get similar to each other, all $b_i$ should get close to each other in feature space. To keep the distances between $b_i$ as small as possible, the variance between all $b_i$ is minimized to obtain $\lambda$:
\begin{equation}\label{squared distance sum2}
\begin{array}{l}
\displaystyle \min ~~ \sum_{i=1}^{N}(B_i\lambda_i-m_i)^{T}(B_i\lambda_i-m_i)  \\[1.5ex]
\displaystyle \mbox{s. t.}~~~  \sum_{j=1}^{N_i}\lambda_{ij}=1,
~~ \lambda_{ij}\geq 0,
\end{array}
\end{equation}
where $\displaystyle m_i=\frac{1}{N}\sum_{i=1}^{N}B_i\lambda_i$.

We define matrix $Z_i$ with $n\times n$ size, $i=1,...,N$. The
diagnose elements in $Z_i$ are 1 at location $(j,j),~~j=N_{i-1}+1,...,N_i$ in
$Z_i$ and 0 elsewhere. It satisfies that
$\displaystyle \sum_{i=1}^{N}Z_i=I$, where $I$ is an identity matrix with
$n\times n$ size.

 \EqRef{squared distance sum2} is rewritten as a standard quadratic programming formulation:
\begin{equation}\label{squared distance sum3}
\begin{array}{l}
\displaystyle \min ~ \lambda^{T} (\sum_{i=1}^{N} [B Z_i-\frac{1}{N}B]^{T}[B Z_i-\frac{1}{N}B]) \lambda  \\[1.5ex]
\displaystyle \mbox{s. t.}~~~~  \sum_{j=1}^{N_i}\lambda_{ij}=1,
~~
 \lambda_{ij}\geq 0,
\end{array}
\end{equation}
where $\displaystyle [B Z_i-\frac{1}{N}B]^{T}[B Z_i-\frac{1}{N}B]$ is
really symmetric and semi-positively definite, so $\displaystyle
(\sum_{i=1}^{N} [B Z_i-\frac{1}{N}B]^{T}[B Z_i-\frac{1}{N}B])$ is also
real symmetric and semi-positive definite. The non-zero part of
$B Z_i\lambda$ is equal to $B_i\lambda_i$. The formulation \EqRef{squared distance sum3} is a convex optimization problem.

To tackle nonlinear separable data, the kernel trick is applied to \EqRef{squared distance sum3}. We reformulate
the $B^\top B$ as the kernel matrix $K$, whose element $K_{ij}$ at $i$-th row and $j$-th column is the kernel function value $ker(x_i,x_j)$, $i,j=1,...,n$.
\EqRef{squared distance sum3} is reformulated:
\begin{equation}\label{kernel squared distance sum}
\begin{array}{l}
\displaystyle  \min ~~ \lambda^{T} (\sum_{i=1}^{N} [K Z_i-\frac{2}{N}K Z_i+\frac{1}{N^2}K]) \lambda  \\[1.5ex]
\displaystyle \mbox{s. t.}~~~~  \sum_{j=1}^{N_i}\lambda_{ij}=1,
~~
 \lambda_{ij}\geq 0.
\end{array}
\end{equation}

As $\lambda$ is known by \EqRef{kernel squared distance sum},
we can learn the concept of positive instance from $\{b_i,...,b_N\}$ using one-class SVM method. The kernel function value between two virtual positive instances is computed:
\begin{equation}\label{kernel definition}
\begin{array}{l}
\displaystyle ker(b_i,b_j)=\sum_{k=1}^{N_i}\sum_{r=1}^{N_j}\lambda_{ik}\lambda_{jr}ker(B_{ik}B_{jr}).  \\[1.5ex]
\end{array}
\end{equation}
\EqRef{one-class SVM primal} is reformulated as the following:
\begin{equation}\label{one-class MI SVM primal}
\begin{array}{l}
\displaystyle    \min_{\alpha}~~  \frac{1}{2}\sum_{i,j}\alpha_{i}\alpha_{j}\sum_{k,r}\lambda_{ik}\lambda_{jr}ker(B_{ik},B_{jr})  \\[1.5ex]

\displaystyle    \mbox{s. t.}~~~~   0\leq \alpha_i \leq
\frac{1}{\nu N}, ~~\sum_i\alpha_i=1.
\end{array}
\end{equation}
As is described previously, $\alpha_i$ in three different ranges implies where $b_i$ locates in feature space.
According to the value of $\alpha_i$, bag $B_i$ can be summarized as two kinds:
\begin{definition}\label{support bag} (\textbf{Support Bag})
For a bag $B_i$, if ~$0 < \alpha_i < \frac{1}{\nu N}$, then $B_i$ is a
``\textbf{Support Bag}''.
\end{definition}
\begin{definition}\label{outlier bag} (\textbf{Outlier Bag})
For a bag $B_i$, if ~$\alpha_i = \frac{1}{\nu N}$, then $B_i$ is an
``\textbf{Outlier Bag}''.
\end{definition}

By solving \EqRef{one-class MI SVM primal}, the prediction function is:
\begin{equation}\label{one-class MI decision function}
\begin{array}{l}
\displaystyle    f(B_i) = \mbox{sign}\big(\max_{r=1,...,N_{i}}l(B_{ir}) \big), \\[1.5ex]
\displaystyle l(B_{ir}) =
\sum_{j=1}^{N}\alpha_j\sum_{k=1}^{N_i}\lambda_{jk} ker(B_{ir},
B_{jk}) - \rho,
\end{array}
\end{equation}
where $B_i$ is the predicted bag. $\rho$ is computed:
\begin{equation}\label{MI rho solution}
\begin{array}{l}
\displaystyle    \rho = \sum_{i=1}^{N}\alpha_i
\sum_{k=1}^{N_i}\sum_{r=1}^{N_j} \lambda_{ik}\lambda_{jr}
ker(B_{ik}, B_{jr}), \\
\displaystyle   ~~B_j \mbox{ satisfies: } 0<\alpha_j<\frac{1}{\nu N}.
\end{array}
\end{equation}

If it holds that $\displaystyle \max_j f(B_{ij})=f(b_i), \forall
i=1,...,N$, $\nu$ satisfies the following theorem:
\begin{theorem} \label{th1}
If $\rho(\neq 0)$ is the solution of \EqRef{one-class MI SVM
primal} and $\displaystyle \max_j f(B_{ij})=f(b_i)(\forall
i=1,...,N$), the following holds: $\nu$ is the upper bound on the
fraction of outlier bags.
\end{theorem}
\begin{proof}
Each bag is converted to one
virtual instance by \EqRef{linear coefficient}. \ThemRef{th1} is easy to be validated
from \emph{proposition} 4 in
\cite{taylor2000support}\cite{taylor2001estimating}.
$\Box$
\end{proof}

If $f(B_{i})=f(b_i)=1$ , at least one instance in $B_i$ must locate on or inside the prediction function boundary(\EqRef{one-class MI decision function}). However, $f(B_{i})=f(b_i)=1$ cannot be
guaranteed completely. In other words, when the virtual instance $b_i=B_i\lambda_i$ falls inside or on the function boundary, all instances in $B_i$ may fall outside. It can be expressed as follows:
\begin{equation}\label{}
\begin{array}{c}
\displaystyle    f(b_i) = +1 \neq \max_{j=1,...,N_i}f(B_{ij})=-1.
\end{array}
\end{equation}
So the assumption $\displaystyle
\max_j f(B_{ij})=f(b_i)$ in \EqRef{th1} does not always hold in
the proof of \ThemRef{th1}.
It causes the upper
bound on outlier bags fraction is larger than $\nu$ sometimes.
To keep the upper bound of outlier bags fraction on $\nu$, our solution is to select the instance $B_{is_i}$ that is closest to function boundary in each bag $B_i$ and add it to the training set. We learn a new prediction function with the updated training set. Now the training set consists of both the selected instances $\{B_is_i\}$ and
the virtual instances $\{b_i\}$. The parameter to control the
number of outlier bag becomes $\frac{\nu}{2}$ since the number of
training instances is $2N$ here.
$B_{is_i}$ is determined by the following criterion:
\begin{equation}\label{replace bag}
\displaystyle    s_i = \arg\max_{j=1,...,N_i}l(B_{ij}).
\end{equation}
The new decision function will replace the result of
\EqRef{one-class MI SVM primal} by solving the following quadratic
programming problem:
\begin{equation}\label{one-class MI SVM primal2}
\begin{array}{l}
\displaystyle    \min_{\alpha} ~~ \frac{1}{2}\sum_{p,q}\alpha_{p}\alpha_{q}\sum_{p,q}ker(x_{p},x_{q})  \\[1.5ex]
\displaystyle    \mbox{s. t.}~~~~   0\leq \alpha_p \leq
\frac{1}{\nu 2N}, ~~\sum_p\alpha_p=1  \\[1.5ex]
\displaystyle  ~~~~~~ x_p, x_q \in \{b_i\}\cup \{B_{is_i}\}, \\[1.5ex]
\displaystyle  ~~~~~~ i=1,...,N,~~~~ p, q = 1,...,2N. \\
\end{array}
\end{equation}
The prediction function from \EqRef{one-class MI SVM primal2} is:
\begin{equation}\label{one-class MI decision function2}
\begin{array}{l}
\displaystyle    f(B_i) = \mbox{sign}\big(\max_{r=1,...,N_{i}}l(B_{ir}) \big) \\
\displaystyle l(B_{ir}) =
\sum_{p=1}^{2N}\alpha_p ker(B_{ir},
x_{p}) - \rho,  \\
\displaystyle    \rho =
\sum_{p=1}^{2N} \alpha_p
ker(x_{p}, x_{q}), ~~x_q \mbox{ satisfies: } 0<\alpha_q<\frac{1}{\nu 2N}.
\end{array}
\end{equation}
Now \EqRef{one-class MI decision function2} is the desired prediction function instead of \EqRef{one-class MI decision function} if $~~\exists i,~~ f(B_i)\neq f(b_i)$.

In \EqRef{one-class MI SVM primal2}, the upper bound on
outlier bags fraction satisfies the following theorem:
\begin{theorem} \label{th:another}
If $\rho(\neq 0)$ is the solution of \EqRef{one-class MI SVM primal2}, the following holds: $\nu$ is the upper bound on the
fraction of outlier bags.
\end{theorem}
\begin{proof}
If $b_i$ falls outside and
$B_{is_i}$ locates inside the prediction function boundary, it holds $\displaystyle \max_j f(B_{ij} = f(B_{is_i}) = 1$). In
the worst case, there are $\frac{\nu}{2}\times 2N=\nu N$ instances
in the set $\{ B_{is_i} \}$ falls outside the prediction function boundary.
This implies that the largest number of outlier bags is $\nu N$.
So in the worst case the upper bound on outlier bag fraction reaches $\nu$ in \EqRef{one-class MI decision function2}.
$\Box$
\end{proof}

\begin{figure}
  \centering
  \includegraphics[scale=0.6]{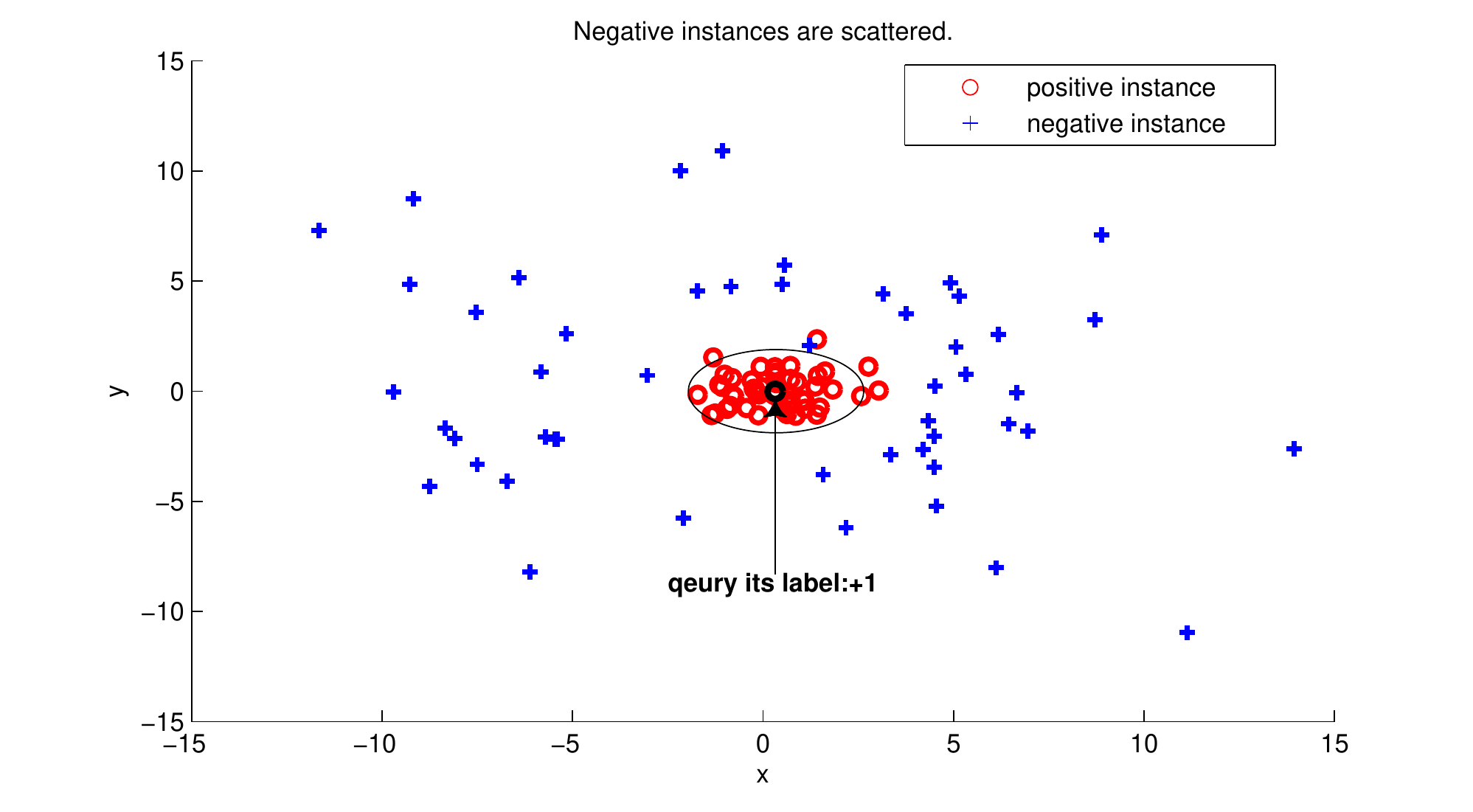}\\
  \caption{The distribution of positive (red) and negative (blue) instances is illustrated. The positive instances are similar to each other and gathered compactly. By comparison, the negative are dissimilar and scattered.}\label{distribution1}
\end{figure}
\begin{figure}
  \centering
  \includegraphics[scale=0.6]{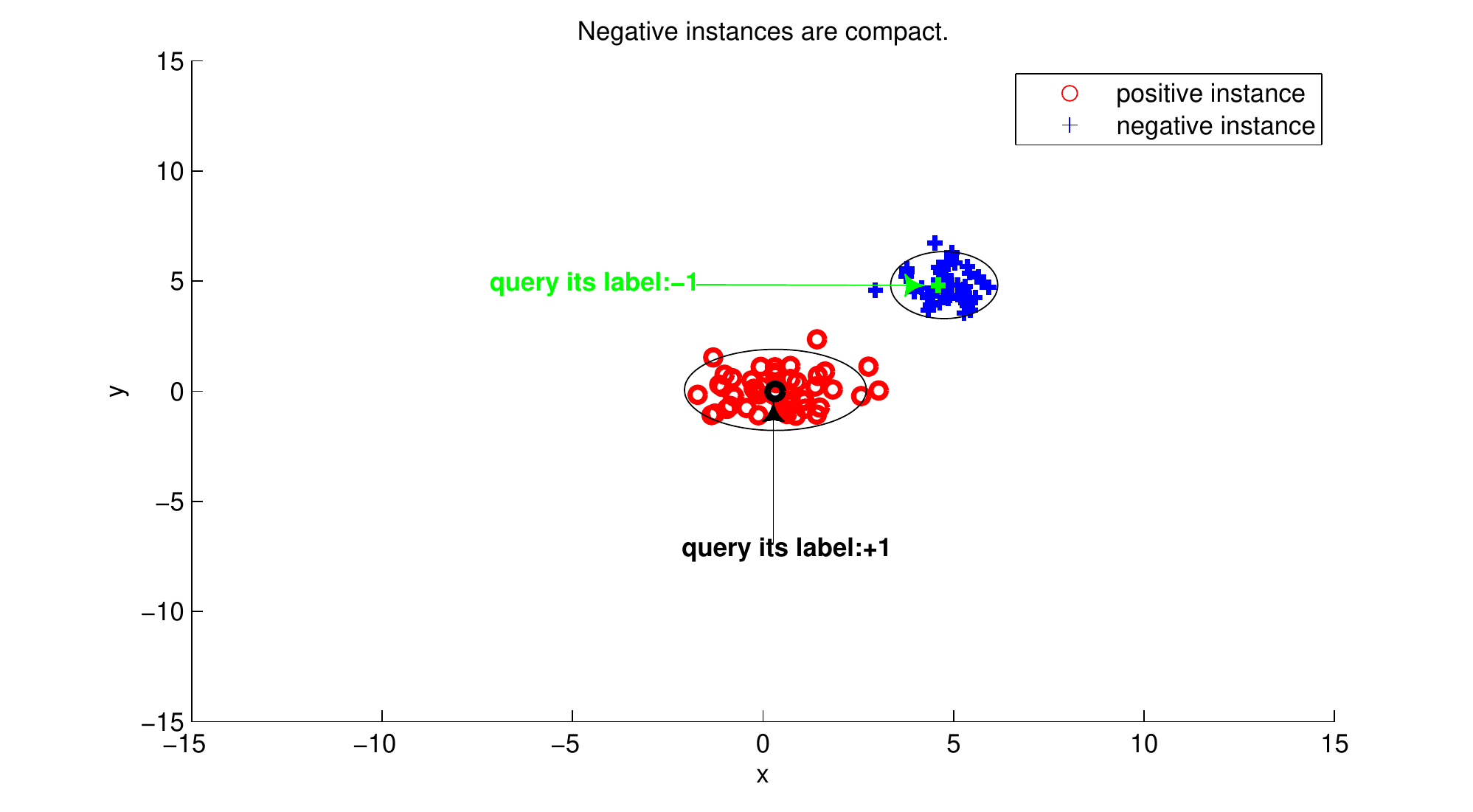}\\
  \caption{The distribution of positive (red) and negative (blue) instances is illustrated. Both the positive and negative instances are clustered. The number of positive instances (red) is equal to that of negative, but negative occupy a smaller circle region. So the negative cluster looks more compact than positive. The query step is needed. The black elliptical circle is positive cluster's boundary. The queried instances (black and green) locate closest to the centers of positive and negative cluster respectively.}\label{distribution2}
\end{figure}

\subsubsection{Query Step}
As is described previously, the positive instances should gather compactly in feature space. On the other hand, the distribution of negative instances is unknown. If the
negative instances are scattered and dissimilar to each other
greatly (\FigRef{distribution1}), we can get the prediction function to enclose most
positive instances in feature space easily. However, negative instances are also clustered sometimes (\FigRef{distribution2}) and even share a higher similarity than positive.
Take face identification as an example, some face images are used to learn a prediction function and each image contains the same background (for example, tree, sky and so
on). If there is a higher similarity between background patches than face ones, the variance of negative instances get smaller than that of positive. The result of training step is to enclose the negative instances with a higher similarity instead of the desired positive cluster.
To avoid this result, our method requires to confirm whether the
cluster enclosed by prediction function boundary is composed of positive or negative instances.
We select an instance $B_{qr}$ characterizes this cluster best to query its label.
The queried instance shares the same label with the most member of this cluster.
If the queried instance is positive, this cluster is positive with a high likelihood and negative otherwise.

The function $l(x)$ is a good measure on the membership of $x$ to the resulted cluster. When $x$ locates outside the cluster, $l(x)<0$ and $l(x)\geq0$ otherwise. The larger $l(x)$ is, the closer $x$ gets to the center of the cluster.
So the instance $B_{qr}$ with maximum $l(x)$ characterizes this cluster best and is selected to query its label:
\begin{equation}\label{instance selection}
\begin{array}{l}
\displaystyle   B_{qr}= \arg\max_{i,j}  l(B_{ij})   \\     
\displaystyle    ~~~~~~~~\mbox{s.t.}~~~~~~ l(B_{ij})\geq 0
\end{array}
\end{equation}
If prediction function boundary is like
a circle, the queried instance usually locates closest to the circle
center than others (\FigRef{distribution2}). If $B_{qr}$ is negative, this cluster should be also negative. Otherwise, this cluster should be positive.
\FigRef{distribution2} illustrates the queried instance location when there are more
than one cluster in feature space.

If the queried instance is negative, we update the current training bags by removing the
instances satisfying $\{B_{ij}| l(B_{ij})\geq 0 \}$ and go back to the training step for a  second instance label query. PMI works in
such a recycle between training and query steps. PMI terminates when the
queried instance is positive or there is one empty bag in training set. PMI is summarized in \AlgRef{alg
mi1learning}.

In \AlgRef{alg mi1learning}, the number of queried
instances satisfies the following:
\begin{theorem} \label{th2}
In \EqRef{one-class MI SVM primal}, if the parameter $\displaystyle
\nu<\frac{1}{N}$, the maximum number of queried instances is
$\displaystyle \min_{i=1,...,N} N_i - 1$ . Otherwise, the maximum
number of queried instance is $\displaystyle
  \lceil\frac{n}{(1-\nu)N}\rceil - 1$.
\end{theorem}
\begin{proof}
According to \ThemRef{th1}, $\nu < \frac{1}{N}$ implies there is
no outlier bag in the current training set. The step \ref{retrain
step} in \AlgRef{alg mi1learning} guarantees that there is at least one
instance $B_{ij}$ satisfying $f(B_{ij})=1$ for every bag
$B_i,~~i=1,...,N$. Thus, At least one instance in each bag must be
removed in every instance removal operation. If one training bag gets empty
after removal, PMI terminates at the step \ref{stop
cond1} in \AlgRef{alg mi1learning}. Therefore, the maximum number of
queried instances is $\min_{i=1,...,N} N_i-1$.

If $\nu \geq \frac{1}{N}$, at least $(1-\nu)N$ instances will be
removed at the step \ref{pmisvm_step end} according to
\ThemRef{th1}. So the maximum number of queried instances is
$\lceil\frac{n}{(1-\nu)N}\rceil - 1$, where $\lceil x\rceil$ denotes the
minimum integer number not smaller than $x$.    $\Box$
\end{proof}

The step \ref{stop cond1} in \AlgRef{alg mi1learning} needs some further
explanations.
If all members of one bag are labeled positive, these instances
belongs to the positive class in that
each training bag contains at least one positive instance.

However, the instance label information is not always available.
When it is difficult to query instance label, our solution is to assume that the dissimilarities between negative instances are larger than positive ones. Thus, \AlgRef{alg
mi1learning} terminates at the step \ref{retrain
step} and outputs the prediction function $f(x)$ without running the query step.
In most real applications, negative instances cover a diverse range of backgrounds or noise. Therefore, it is almost impossible that negative instances share a higher similarity than positive. It is high likely to achieve the concept on positive instance after only one training step.
The experimental results in
the next section illustrate that our hypothesis is applicable for most of
data sets. It also implies that the maximum query number in
\ThemRef{th2} is not a tight bound. The number of queried instance
label usually is far smaller than the result in
\ThemRef{th2} and close to 1 mostly. It suggests that the cost of
instance label query is very limited, which is much lower than that
of labeling a large number of negative bags.

\begin{algorithm}
\caption{PMI} \label{alg mi1learning}
\begin{algorithmic}[1]
\Require

the training bags $\mathcal{B}=\{B_i,~~i=1,...,N\}$;

kernel function parameters;

$\nu$;

\State Solve \EqRef{kernel squared distance sum} to get $\lambda$;
\label{lambda step}

\State Solve \EqRef{one-class MI SVM primal} to get a prediction function
$f_0 (\EqRef{one-class MI decision function})$. Set $f=f_0$; \label{train step}

\State If there exists one bag $B_i$ satisfying both $\alpha_i<\frac{1}{\nu N}$ and $\max_{j\in 1,...,N_i}f(B_{ij})=-1$, select an instance $B_{is_i}$ by \EqRef{instance
selection} from every training bag $B_i, ~i=1,...,N$;

\State Solve \EqRef{one-class MI SVM primal2} to get a prediction function $f_1 (\EqRef{one-class MI decision function2})$. Set $f=f_1$; \label{retrain step}

\State If no instance label information is available, PMI
terminates; \label{no query}

\State If there is one bag $B_i,~i=1,...,N$ satisfying that ${\forall j\in 1,...,N_i}~~ f(B_{ij}) = 1$, PMI terminates; \label{stop cond1}

\State Select the most certain instance by \EqRef{instance
selection} to query its label;

\State If the queried label is positive, PMI ends. Otherwise, update the training set $\mathcal{B}$ by removing the instances $\forall i,j~~B_{ij}$ satisfying $f(B_{ij})=1$.
Go back to the step
\ref{lambda step}; \label{pmisvm_step end}

\Ensure

prediction function $f$;
\end{algorithmic}
\end{algorithm}

\subsection{Time Complexity Analysis}
The time complexity analysis of PMI involves \EqRef{kernel squared
distance sum}, \EqRef{one-class MI SVM primal}, \EqRef{one-class MI
SVM primal2} and the number of the queried instances. These three
objective functions \EqRef{kernel squared distance sum},
\EqRef{one-class MI SVM primal} and \EqRef{one-class MI SVM primal2}
are three quadratic programming problems with the time complexities
$\textit{O}( n^3 )$, $\textit{O}( N^3)$, $\textit{O}( N^3 )$
respectively. Since $n>N$ in general, the term $\textit{O}( n^3
)$ dominates the time complexity of PMI. According to \ThemRef{th2},
the number of queried instances is related to $\displaystyle
\frac{n}{N}$. So the time complexity is $\displaystyle
\textit{O}( \frac{n^4}{N} )$.
But in real applications, the queried instance number is much smaller than the theoretical result in \ThemRef{th2} and can be approximated to a constant. The time complexity of PMI is close to $\displaystyle
\textit{O}(n^3)$.

\section{Experiments}
This section presents the experimental results on five benchmark and one face image data sets. We evaluate PMI compared the traditional MIL algorithms. To solve the quadratic programming problem in \EqRef{one-class SVM primal}, the optimization toolbox ``Mosek''
\cite{Mosek} is used.

\subsection{Benchmark Data Sets}
Five benchmark data sets are used in our experiment. They are
Musk1, Musk2, Elephant, Fox and Tiger, which were used frequently to
test new MIL algorithms in previous
studies\cite{dietterich1997solving,fung2007multiple,andrews2003support}. Both Musk1 and Musk2 come from UCI data set web
site \cite{ucidataset}. And the other three derive from Corel image
set \cite{Corel}. The details of five data sets are described in
\TabRef{tb1}.

\begin{table*}[htbp]
\centering
 \caption{five datasets details}  \label{tb1}
 \begin{tabular}{|c|c|c|c|c|c|}
\hline
  data set & bag number (positive) & mean bag size & dimensionality \\
\hline
  Musk1 & 92(47) & 5.17 & 166 \\
\hline
  Musk2 & 102(36) & 64.69 & 166 \\
\hline
  Elephant & 200(100) & 6.96 & 230 \\
\hline
  Fox   &  200(100) &  6.90 &  230 \\
\hline
  Tiger &  200(100) &  6.10 &  230  \\
\hline
 \end{tabular}
\end{table*}

The accuracies of PMI on five benchmark data sets are recorded in \TabRef{tb2}. To compare with PMI, we also provide the performances \cite{fung2007multiple,andrews2003support} of several other traditional MIL algorithms .
These traditional MIL algorithms are mi-SVM \cite{andrews2003support}, MI-SVM\cite{andrews2003support},
EM-DD\cite{zhang2002dd} and MICA\cite{fung2007multiple}.

The same with previous MIL studies, 10-fold cross validation
strategy is used to get the accuracies of PMI. Both positive and
negative bags are divided into 10 folds randomly and one fold
positive and negative bags are selected as testing set. The remaining
9 fold positive bags (without negative ones) are the training set.
The difference between that of PMI and 10-fold cross validation used
in traditional MIL studies is that the training set of PMI does not
contain 9 fold negative bags.
\begin{table*}[htbp]
\centering
 \caption{The numbers of positive bags used for training in PMI and other MIL algorithms are listed.}  \label{tb:pos_bags}
 \begin{tabular}{|c|p{1.2cm}<{\centering}|p{1.2cm}<{\centering}|c|p{1.2cm}<{\centering}|p{1.2cm}<{\centering}|}
\hline
     & Musk1 & Musk2 & Elephant & Fox & Tiger \\
\hline
  \textbf{PMI}  & \textbf{47} & \textbf{36} & \textbf{100} & \textbf{100} & \textbf{100}\\
\hline
  mi-SVM & 92 & 102 & 200 & 200 & 200\\
\hline
  MI-SVM & 92 & 102 & 200 & 200 & 200\\
\hline
  MICA   & 92 & 102 & 200 & 200 & 200\\
\hline
  EM-DD  & 92 & 102 & 200 & 200 & 200\\
\hline
 \end{tabular}
\end{table*}

\begin{table*}[htbp]
\centering
 \caption{The classification accuracies of five algorithms on the five data sets are presented. The results are percentage of correctly predicted bags. These accuracies are average results of 10 runs.}  \label{tb2}
 \begin{tabular}{|c|c|c|c|c|c|}
\hline
     & Musk1 & Musk2 & Elephant & Fox & Tiger \\
\hline
  PMI  & 79.1$\pm$1.1 & 85.1$\pm$1.2 & 78.3$\pm$0.7 & 57.9$\pm$1.5 & 52.5$\pm$1.0\\
\hline
  mi-SVM & 87.4 & 83.6 & 80.0 & 57.9 & 78.9\\
\hline
  MI-SVM & 77.9 & 84.3 & 73.1 & 58.8 & 66.6\\
\hline
  MICA   & 84.4 & 90.5 & 82.5 & 62.0 & 82.0\\
\hline
  EM-DD  & 84.8 & 84.9 & 78.3 & 56.1 & 72.1\\
\hline
 \end{tabular}
\end{table*}

\TabRef{tb:pos_bags} shows the numbers of positive bags used for training in PMI and other traditional MIL algorithms. \TabRef{tb2} reports the accuracies of five approaches on five benchmark data sets.
These accuracies are the average of 10 runs. PMI
terminates at the step \ref{no query} of \AlgRef{alg mi1learning}
and returns $f(x)$ due to no available instance label of
these five data sets. PMI achieves close results
to those of the other traditional MIL methods mostly. The
result on Musk2 is better than mi-SVM, MI-SVM, EM-DD, but worse than
MICA. Musk2 has only 39 positive bags, which are obviously smaller
than the negative bag number 63. So it is exciting that PMI achieves
the similar performance using much fewer training bags (only
positive bags) to those of other MIL approaches. PMI achieves a little
lower accuracy on Musk1 than mi-SVM, MICA and EM-DD, but a little larger than
MI-SVM. The results of PMI are
close to those of the other MIL methods on Elephant and Fox.

According to the performance comparison in \TabRef{tb2}, if the traditional MIL methods achieve a high accuracy on one data set, PMI can also obtain similar result.
But when the performance of traditional MIL method becomes very low(only about 60\% accuracy on Fox data set), PMI will get a worse result than those of traditional MIL methods.
We explain the phenomenon as follows.
If there is a enough large margin and no overlap between positive and negative instances, the traditional MIL methods with positive and negative training bags usually get a perfect accuracy. Due to no overlap between positive and negative instances, PMI with only positive training bags can describe the positive instance distribution accurately and also achieves satisfying performance.
However, if there exists a big overlap between positive and negative, we will get a poor result though both positive and negative bags are used as the training set. So the result will degrade furthermore with only positive training bags.

\subsection{Face Identification}
This experiment shows result of PMI on face identification. The image set
\cite{aam} is used to evaluate PMI. The task of face identification is
predict whether an image contains face(s) or not. Each of these face
images contains one person's face and the similar green background.
This image set collects 37 persons' faces from various views under
different light conditions. Three face samples for each person are
recorded. So this face image set includes totally 111 face images.
Every person has various expressions. These 37 persons include
various types: male and female, young and old. Each image is
$640\times480$ size with JPEG file format. No non-face image is
provided in this image set.
\FigRef{face images} illustrates some examples of this image set. Much
more non-face images than face are needed in previous MIL studies. But no non-face image is provided in this data set.

\begin{figure}
  \centering
  \includegraphics[scale=0.5]{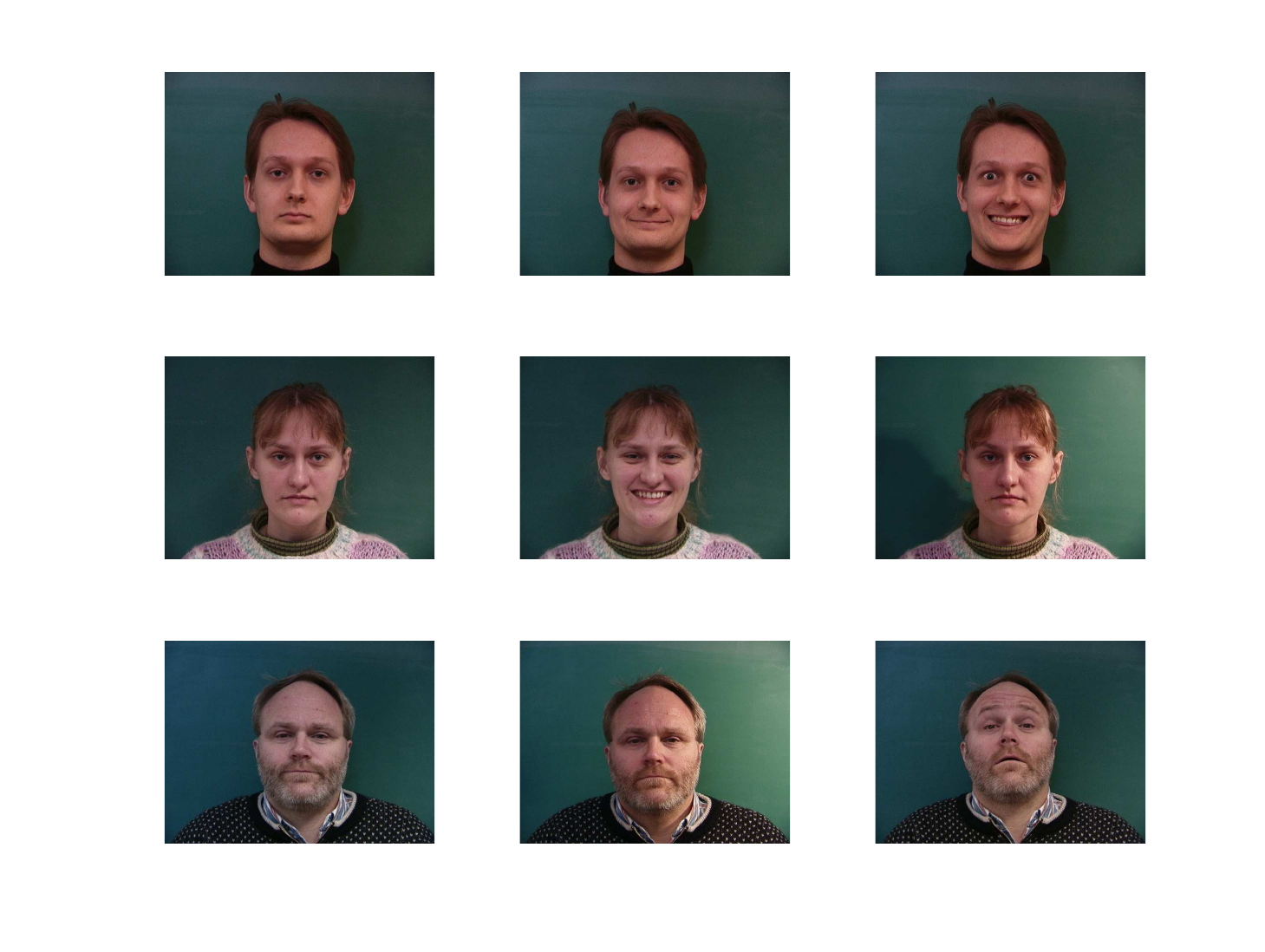}\\
  \caption{The face images from three persons in this data set are shown. Each row corresponds one person's images.}\label{face images}
\end{figure}

The feature extraction method in \cite{chen2006miles} is used in our
experiments. Only the brief introduction is provided. Please read
\cite{chen2004image} for more details. The first is to split
each image into $4\times4$ blocks and compute L, U, V mean values of
each block in LUV color space. Secondly, 4-order wavelet
transformation is applied to L value on each block to compute mean values of LH,
HL, HH bands. Each block is transformed into a
6-dimension (L,U,V,LH,HL,HH) feature vector and scaled into $[0,1]$. We use
the toolkit ``JSEG''\cite{jseg2001deng} to segment each image into
regions. Each region is represented by the mean feature vector of
the blocks in this region. So one region corresponds to an instance. An
image is associated with a bag. The average instance number of
positive bags is 7.8.

\begin{table}[htbp]
\centering
 \caption{The numbers of positive bags used for training in PMI and other MIL algorithms are listed.}  \label{tb_face_pos_bags}
 \begin{tabular}{|c|p{1.5cm}<{\centering}|p{1.8cm}<{\centering}|p{1.8cm}<{\centering}|p{1.5cm}<{\centering}|p{1.5cm}<{\centering}|}
\hline
algorithm & \textbf{PMI} & mi-SVM & MI-SVM & DD & EM-DD\\
\hline
training bag number & \textbf{111} & 211 & 211 &  211 & 211\\
\hline
 \end{tabular}
\end{table}

\begin{table}[htbp]
\centering
 \caption{The classification accuracies of both traditional MIL methods and PMI.}  \label{tb_face_accuracy}
 \begin{tabular}{|c|p{1.5cm}<{\centering}|p{1.8cm}<{\centering}|p{1.8cm}<{\centering}|p{1.5cm}<{\centering}|p{1.5cm}<{\centering}|} 
\hline
algorithm & PMI & mi-SVM & MI-SVM & DD & EM-DD\\  %
\hline
accuracy & 90.4$\pm$1.9 & 89.6$\pm$1.3 & 86.5$\pm$1.8 &  85.2$\pm$2.4 & 97.2$\pm$1.6 \\  %
\hline
 \end{tabular}
\end{table}

We select 100 background images(for the balance between positive and negative bags)
randomly from
\cite{CaltechImage,LFeifei2004learning} as negative bags
for testing and training in other traditional MIL methods. These
100 negative images contain rich contents and backgrounds.
The average instance number of negative bags is about 19. The
feature extraction method is the same as that of the face images.
The numbers of training bags in PMI and other MIL algorithms are reported in \TabRef{tb_face_pos_bags}.
The 5-fold cross validation accuracies of both the traditional MIL
algorithms and PMI are reported in \TabRef{tb_face_accuracy}. Each
accuracy is the average of 10 runs. PMI achieves close performances to the other MIL algorithms. The accuracy of mi-SVM keeps very close to that of PMI. In contrast, MI-SVM and DD fall a little behind the other three methods.

\begin{table}[htbp]
\centering
 \caption{The number of queried instances under different $\nu$ and $\gamma$ combinations}  \label{tb3}
 \begin{tabular}{|p{1.6cm}<{\centering}|p{1.2cm}<{\centering}|p{1.2cm}<{\centering}|p{1.2cm}<{\centering}|p{1.2cm}<{\centering}|p{1.2cm}<{\centering}|}  
\hline
 \backslashbox{$\nu$}{$\gamma$}  & 60 & 70 & 80 & 90 & 100 \\
\hline
  0.01  & 1 & 1 & 1 & 1 & 1\\
\hline
  0.05  & 1 & 1 & 1 & 1 & 1\\
\hline
  0.1   & 1 & 1 & 1 & 1 & 1\\
\hline
  0.2   & 1 & 1 & 1 & 1 & 1\\
\hline
  0.3   & 1 & 1 & 1 & 1 & 1\\
\hline
  0.5   & 1 & 1 & 1 & 1 & 1\\
\hline

 \end{tabular}
\end{table}

To validate \ThemRef{th2}, we use all face images as the training set.
\TabRef{tb3} reports the number of queried instances under different
parameter combinations of $\gamma$ and $\nu$.
 According to \ThemRef{th2}, the maximum number of queried instances is larger than
$\lceil7.81\rceil-1=7$ when $\nu\geq \frac{1}{N}=\frac{1}{111}$.
7.81 is the average instance number of positive bags. The largest
number in \TabRef{tb3} is much smaller than
7, which demonstrates \ThemRef{th2} is correct.

\section{Conclusion}
This paper proposes a new algorithm PMI, which learns a
prediction function with only positive bags. Our proposed algorithm is to reduce the label cost significantly compared with previous MIL algorithms without much accuracy loss.
A detailed theoretic analysis is also provided. In most real applications, the queried instance number is much smaller than the theoretical result in \ThemRef{th2} and can be approximated to a constant. The time complexity of PMI is approximated as $\displaystyle
\textit{O}(n^3)$.
 Experimental results illustrate that
PMI usually achieves close performance to those of traditional MIL
methods mostly especially when traditional MIL algorithms can predict bag label with a high accuracy.
In the future work, we try to apply some dimensional reduction approaches \cite{hou2009stable,hou2010multiple,hou2014multiple} on our proposed method to improve the efficiency.

\section*{Acknowledgement}
Acknowledge here.

\section*{Appendix} Appendix here.

%









\bibliographystyle{abbrv}
\bibliography{my}
\end{document}